\documentclass{article}
\pdfoutput=1

\usepackage[final,nonatbib]{nips_2016}

\usepackage[utf8]{inputenc} 
\usepackage[T1]{fontenc}    
\usepackage{hyperref}       
\usepackage{url}            
\usepackage{booktabs}       
\usepackage{amsfonts}       
\usepackage{nicefrac}       
\usepackage{microtype}      

\usepackage{subfigure}
\usepackage{todonotes}
\usepackage{amsmath}
\usepackage{amssymb}
\usepackage{xspace}
\usepackage{multirow}
\usepackage{mathtools}

\usepackage{paralist}
\usepackage{enumitem}
\usepackage[thmmarks,hyperref,amsthm,amsmath]{ntheorem}

\newtheorem{definition}{Definition}
\newtheorem{lemma}{Lemma}
\newtheorem{proposition}{Proposition}
\newtheorem{corollary}{Corollary}
\newtheorem{theorem}{Theorem}

\newcommand{\lab}[0]{\tau}
\newcommand{\norm}[1]{\left\lVert#1\right\rVert}
\newcommand{\tp}[0]{\top}
\newcommand{\X}{\ensuremath{\mathcal{X}}\xspace}
\newcommand{\V}{\ensuremath{\mathcal{V}}\xspace}

\newcommand{\G}{\ensuremath{\mathcal{G}}\xspace}
\newcommand{\Hilb}{\ensuremath{\mathcal{H}}\xspace}
\newcommand{\bbN}[0]{\ensuremath{\mathbb{N}}\xspace}
\newcommand{\bbR}[0]{\ensuremath{\mathbb{R}}\xspace}
\newcommand{\bbRnn}[0]{\ensuremath{\mathbb{R}_{\geq0}}\xspace}
\newcommand{\Assign}{\ensuremath{\mathfrak{B}}\xspace}

\makeatletter
\renewcommand{\@notice}{%
  \enlargethispage{2\baselineskip}%
  \@float{noticebox}[b]%
    \footnotesize\vspace{1em}%
  \end@float%
}
\makeatother

\let\vec\mathbf

\DeclareMathOperator{\img}{img}

\graphicspath{{figures/}}

\usepackage[font=small]{caption}
\title{On Valid Optimal Assignment Kernels and Applications to Graph Classification}

\author{
  Nils M. Kriege\\
  Department of Computer Science\\
  TU Dortmund, Germany\\
  \texttt{nils.kriege@tu-dortmund.de}
  \And 
  Pierre-Louis Giscard\\
  Department of Computer Science\\
  University of York, UK\\
  \texttt{pierre-louis.giscard@york.ac.uk}
  \And
  Richard C. Wilson\\
  Department of Computer Science\\
  University of York, UK\\
  \texttt{richard.wilson@york.ac.uk}
}

\begin{document}

\maketitle

\begin{abstract}
The success of kernel methods has initiated the design of novel positive 
semidefinite functions, in particular for structured data.
A leading design paradigm for this is the convolution kernel, which decomposes 
structured objects into their parts and sums over all pairs of parts.
Assignment kernels, in contrast, are obtained from an optimal bijection between 
parts, which can provide a more valid notion of similarity.
In general however, optimal assignments yield indefinite functions, which
complicates their use in kernel methods. We characterize a class of base 
kernels used to compare parts that guarantees positive semidefinite optimal assignment kernels.
These base kernels give rise to hierarchies from which the optimal assignment 
kernels are computed in linear time by histogram intersection. We apply these 
results by developing the Weisfeiler-Lehman optimal assignment kernel for graphs. 
It provides high classification accuracy on widely-used benchmark data sets 
improving over the original Weisfeiler-Lehman kernel.
\end{abstract}

\section{Introduction}
The various existing kernel methods can conveniently be applied to any type of 
data, for which a kernel is available that adequately measures the similarity
between any two data objects. 
This includes structured data like images~\cite{Barla2003,Boughorbel2005,Grauman2007},
3d shapes~\cite{Bai2015}, chemical compounds~\cite{Frohlich2005} and 
proteins~\cite{Borgwardt2005a}, which are often represented by graphs. 
Most kernels for structured data decompose both objects and add up the pairwise similarities between
their parts following the seminal concept of convolution kernels proposed by 
Haussler~\cite{Haussler1999}. In fact, many graph kernels can be seen as 
instances of convolution kernels under different decompositions~\cite{Vishwanathan2010}.

A fundamentally different approach with good prospects is to \emph{assign}
the parts of one objects to the parts of the other, such that the total 
similarity between the assigned parts is maximum possible. 
Finding such a bijection is known as \emph{assignment problem} and well-studied
in combinatorial optimization~\cite{Burkard2012}.
This approach has been successfully applied to graph comparison, e.g., in general 
graph matching~\cite{Gori2005,Riesen2009a} as well as in kernel-based
classification~\cite{Frohlich2005,Schiavinato2015,Bai2015}.
In contrast to convolution kernels, assignments establish structural 
correspondences and thereby alleviate the problem of diagonal dominance at the 
same time.
However, the similarities derived in this way are not necessarily positive 
semidefinite (p.s.d.)~\cite{Vert2008,Vishwanathan2010} and hence do not give 
rise to valid kernels, severely limiting their use in kernel methods. 

Our goal in this paper is to consider a particular class of base kernels which 
give rise to valid assignment kernels. In the following we use the term valid to 
mean a kernel which is symmetric and positive semidefinite. 
We formalize the considered problem:
Let $[\X]^n$ denote the set of all $n$-element subsets of a set $\X$ and 
$\Assign(X,Y)$ the set of all bijections between $X,Y$ in $[\X]^n$ for 
$n \in \bbN$. 
We study the \emph{optimal assignment kernel} $K_\Assign^k$ on $[\X]^n$ defined as
\begin{equation}\label{eq:assignment_kernel}
 K_\Assign^k(X,Y) = \max_{B \in \Assign(X,Y)} W(B), \quad\text{ where } W(B) = \sum_{(x,y) \in B} k(x,y)
\end{equation}
and $k$ is a \emph{base kernel} on \X. 
For clarity of presentation we assume $n$ to be fixed. In order to apply the 
kernel to sets of different cardinality, we may fill up the smaller set by new 
objects $z$ with $k(z,x)=0$ for all $x \in \X$ without changing the result.

\paragraph{Related work.}
Correspondence problems have been extensively studied in object recognition, 
where objects are represented by sets of features often called \emph{bag of words}.
Grauman and Darrell proposed the \emph{pyramid match kernel} that seeks to
approximate correspondences between points in $\bbR^d$ by employing a 
space-partitioning tree structure and counting how often points fall into the 
same bin~\cite{Grauman2007}. An adaptive partitioning with non-uniformly shaped 
bins was used to improve the approximation quality in high dimensions~\cite{Grauman2007a}.

For non-vectorial data, Fröhlich et al.~\cite{Frohlich2005} proposed kernels for 
graphs derived from an optimal assignment between their vertices and applied the 
approach to molecular graphs. However, it was shown that the resulting similarity
measure is not necessarily a valid kernel~\cite{Vert2008}.
Therefore, Vishwanathan et al.~\cite{Vishwanathan2010} proposed a theoretically 
well-founded variation of the kernel, which essentially replaces the 
$\max$-function in Eq.~\eqref{eq:assignment_kernel} by a soft-max function.
Besides introducing an additional parameter, which must be chosen carefully to 
avoid numerical difficulties, the approach requires the evaluation of a sum over all 
possible assignments instead of finding a single optimal one. 
This leads to an increase in running time from cubic to factorial, which is 
infeasible in practice.
Pachauri et al.~\cite{Pachauri2013} considered the problem of finding optimal assignments between 
multiple sets. The problem is equivalent to finding a permutation of the elements 
of every set, such that assigning the $i$-th elements to each other yields an 
optimal result.
Solving this problem allows the derivation of valid kernels between pairs of sets with a
fixed ordering. This approach was referred to as \emph{transitive assignment kernel}
in~\cite{Schiavinato2015} and employed for graph classification.
However, this does not only lead to non-optimal assignments between individual 
pairs of graphs, but also suffers from high computational costs.
Johansson and Dubhashi~\cite{Johansson2015} derived kernels from 
optimal assignments by first sampling a fixed set of so-called \emph{landmarks}. 
Each data point is then represented by a feature vector, where each component is 
the optimal assignment similarity to a landmark.

Various general approaches to cope with indefinite kernels have been proposed, 
in particular, for support vector machines (see \cite{Loosli2015} and references therein).
Such approaches should principally be used in applications, where similarities 
cannot be expressed by positive semidefinite kernels.

\paragraph{Our contribution.}
We study optimal assignment kernels in more detail and investigate which base 
kernels lead to valid optimal assignment kernels. We characterize a specific 
class of kernels we refer to as \emph{strong} and show that strong kernels are 
equivalent to kernels obtained from a hierarchical partition of the domain 
of the kernel. We show that for strong base kernels the optimal assignment 
(i) yields a valid kernel; and (ii) can be computed in linear time given the
associated hierarchy.
While the computation reduces to histogram intersection similar to the
pyramid match kernel~\cite{Grauman2007}, our approach is in no way restricted 
to specific objects like points in $\bbR^d$. We demonstrate the versatility
of our results by deriving novel graph kernels based on optimal assignments, 
which are shown to improve over their convolution-based counterparts.
In particular, we propose the Weisfeiler-Lehman optimal assignment kernel, which
performs favourable compared to state-of-the-art graph kernels on a wide range 
of data sets.

\section{Preliminaries}
Before continuing with our contribution, we begin by introducing some key notation
for kernels and trees which will be used later.
A \emph{(valid) kernel} on a set \X is a function 
$k : \X \times \X \to \mathbb{R}$ such that there is a real 
Hilbert space $\Hilb$ and a mapping $\phi : \X \to \Hilb$ 
such that $k(x,y) = \langle \phi(x), \phi(y) \rangle$ for all $x,y$ in \X, 
where $\langle \cdot, \cdot \rangle$ denotes the inner product of $\Hilb$. 
We call $\phi$ a \emph{feature map}, and $\Hilb$ a \emph{feature space}. 
Equivalently, a function $k : \X \times \X \to \mathbb{R}$ is a kernel
if and only if for every subset $\{x_1,\dots,x_n\} \subseteq \X$ the 
$n\times n$ matrix defined by $[m]_{i,j} = k(x_i,x_j)$ is p.s.d.
The Dirac kernel $k_\delta$ is defined by $k_\delta(x,y) = 1$, if $x=y$ and $0$ 
otherwise.

We consider simple undirected graphs $G=(V,E)$, where $V(G) = V$ is the set of 
\emph{vertices} and $E(G)=E$ the set of \emph{edges}. An edge $\{u,v\}$ is for
short denoted by $uv$ or $vu$, where both refer to the same edge.
A graph with a unique path between any two vertices is a \emph{tree}.
A \emph{rooted tree} is a tree $T$ with a distinguished vertex $r \in V(T)$ 
called \emph{root}.
The vertex following $v$ on the path to the root $r$ is called \emph{parent} of 
$v$ and denoted by $p(v)$, where $p(r) = r$.
The vertices on this path are called \emph{ancestors} of $v$ and the \emph{depth} 
of $v$ is the number of edges on the path.
The \emph{lowest common ancestor} ${\rm LCA}(u,v)$ of two vertices $u$ and $v$ 
in a rooted tree is the unique vertex with maximum depth that is an ancestor of
both $u$ and $v$.

\section{Strong kernels and hierarchies}
In this section we introduce a restricted class of kernels that will later turn 
out to lead to valid optimal assignment kernels when employed as base kernel.
We provide two different characterizations of this class, one in terms of an
inequality constraint on the kernel values, and the other by means of a hierarchy
defined on the domain of the kernel. The latter will provide the basis for our 
algorithm to compute valid optimal assignment kernels efficiently.

We first consider similarity functions fulfilling the requirement that for any 
two objects there is no third object that is more similar to each of them than 
the two to each other.
We will see later in Section~\ref{sec:strong:feature_maps} that every such 
function indeed is p.s.d. and hence a valid kernel.
\begin{definition}[Strong Kernel]\label{def:strong}
 A function 
 $k : \X \times \X \to \bbRnn$ is called \emph{strong kernel} 
 if $k(x,y) \geq \min\{k(x,z), k(z,y)\}$ for all $x,y,z \in \X$.
\end{definition}
Note that a strong kernel requires that every object is most similar to itself, 
i.e.,  $k(x,x) \geq k(x,y)$ for all $x,y \in \X$.

In the following we introduce a restricted class of kernels that is derived
from a hierarchy on the set \X. As we will see later in 
Theorem~\ref{thm:strong_hierarchy} this class of kernels is equivalent to 
strong kernels according to Definition~\ref{def:strong}.
Such hierarchies can be systematically constructed on sets of arbitrary objects 
in order to derive strong kernels.
We commence by fixing the concept of a hierarchy formally.
Let $T$ be a rooted tree such that the leaves of $T$ are the elements of 
\X.
Each inner vertex $v$ in $T$ corresponds to a subset of \X comprising 
all leaves of the subtree rooted at $v$. Therefore the tree $T$ defines a 
family of nested subsets of \X. 
Let $w : V(T) \to \bbRnn$ be a weight function such that $w(v) \geq w(p(v))$ for
all $v$ in $T$. We refer to the tuple $(T,w)$ as a \emph{hierarchy}. 

\begin{definition}[Hierarchy-induced Kernel]\label{def:h_induced}
 Let $H=(T,w)$ be a hierarchy on \X, then the function defined as
 $k(x,y) = w({\rm LCA}(x,y))$ for all $x,y$ in \X is the 
 kernel on \X \emph{induced} by $H$. 
\end{definition}

We show that Definitions~\ref{def:strong} and~\ref{def:h_induced} characterize 
the same class of kernels.

\begin{lemma}\label{prop:strong_hierarchy:left}
 Every kernel on \X that is induced by a hierarchy on \X is strong.
\end{lemma}
\begin{proof}
Assume there is a hierarchy $(T,w)$ that induces a kernel $k$ that is not strong.
Then there are $x,y,z \in \X$ with $k(x,y) < \min\{k(x,z), k(z,y)\}$ and three 
vertices $a={\rm LCA}(x,z)$, $b={\rm LCA}(z,y)$ and $c={\rm LCA}(x,y)$ with 
$w(c) < w(a)$ and $w(c) < w(b)$. 
The unique path from $x$ to the root contains $a$ and the path from $y$ to the 
root contains $b$, both paths contain $c$.
Since weights decrease along paths, the assumption implies that $a,b,c$ are 
pairwise distinct and $c$ is an ancestor of $a$ and $b$. 
Thus, there must be a path from $z$ via $a$ to $c$ and another path from $z$ 
via $b$ to $c$.
Hence, $T$ is not a tree, contradicting the assumption.
\end{proof}

We show constructively that the converse holds as well.
\begin{lemma}\label{prop:strong_hierarchy:right}
 For every strong kernel $k$ on \X there is a hierarchy on \X that induces $k$.
\end{lemma}
\begin{figure}
  \centering
  \null\hfill
  \subfigure[$H_i$]{\label{fig:hierarchy_proof:old}
    \includegraphics{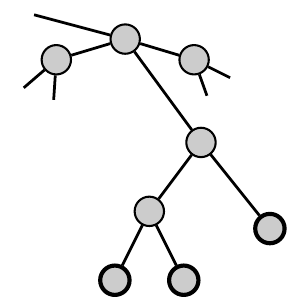}
  }\hfill
  \subfigure[$H_{i+1}$ for $\mathcal{B}=\{b_1,b_2,b_3\}$]{\label{fig:hierarchy_proof:new_1}
    \includegraphics{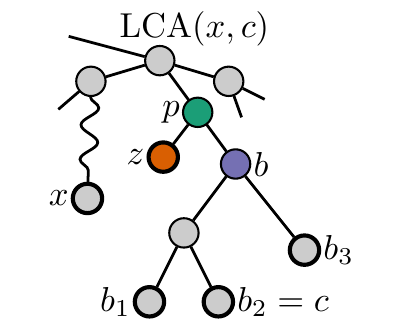}
  }\hfill
  \subfigure[$H_{i+1}$ for $|\mathcal{B}|=1$]{\label{fig:hierarchy_proof:new_2}
    \includegraphics{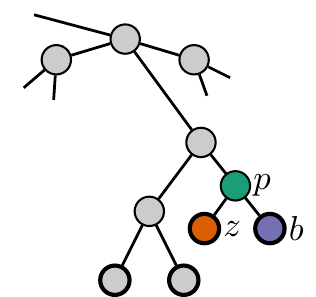}
  }
  \hfill\null
  \caption{Illustrative example for the construction of the hierarchy on $i+1$ objects~\subref{fig:hierarchy_proof:new_1},~\subref{fig:hierarchy_proof:new_2} from the hierarchy on $i$ objects~\subref{fig:hierarchy_proof:old} following the procedure used in the proof of Lemma~\ref{prop:strong_hierarchy:right}. The inserted leaf $z$ is highlighted in
  red, its parent $p$ with weight $w(p) = k_{\max}$ in green and $b$ in blue, respectively.}
  \label{fig:hierarchy_proof}
\end{figure}
\begin{proof}[Proof (Sketch)]
 We incrementally construct a hierarchy on \X that induces $k$ by successive insertion
 of elements from \X. In each step the hierarchy induces $k$ restricted to the inserted 
 elements and eventually induces $k$ after insertion of all elements.
 Initially, we start with a hierarchy containing just one element $x \in \X$  with
 $w(x) = k(x,x)$.
 The key to all following steps is that there is a unique way to extend the 
 hierarchy:
 Let $\X_{i} \subseteq \X$ be the first $i$ elements in the order of insertion and 
 let $H_i=(T_i,w_i)$ be the hierarchy after the $i$-th step. A leaf 
 representing the next element $z$ can be grafted onto $H_i$ to form a hierarchy 
 $H_{i+1}$ that induces $k$ restricted to $\X_{i+1}= \X_i \cup \{z\}$. 
 Let $\mathcal{B}=\{x\in\X_{i}:\,k(x,z)=k_{\text{max}}\}$, where 
 $k_{\text{max}}=\max_{y\in \X_{i}} k(y,z)$.
 There is a unique vertex $b$, such that $\mathcal{B}$ are the leaves of the subtree 
 rooted at $b$, cf. Fig.~\ref{fig:hierarchy_proof}.
 We obtain $H_{i+1}$ by inserting a new vertex $p$ with child $z$ into $T_i$, such that 
 $p$ becomes the parent of $b$, cf. Fig.~\ref{fig:hierarchy_proof:new_1},~\subref{fig:hierarchy_proof:new_2}.
 We set $w_{i+1}(p)=k_{\max}$, $w_{i+1}(z)=k(z,z)$ and $w_{i+1}(x)=w_{i}(x)$ for all $x \in V(T_i)$.
 Let $k'$ be the kernel induced by $H_{i+1}$. Clearly, $k'(x,y)=k(x,y)$ for all $x,y \in \X_i$.
 According to the construction $k'(z,x) = k_{\text{max}} = k(z,x)$ for all $x \in \mathcal{B}$.
 For all $x \notin \mathcal{B}$ we have ${\rm LCA}(z,x) = {\rm LCA}(c,x)$ for any 
 $c \in \mathcal{B}$, see Fig.~\ref{fig:hierarchy_proof:new_1}.
 For strong kernels $k(x,c)\geq \min\{k(x,z),k(z,c)\}= k(x,z)$ and $k(x,z)\geq \min\{k(x,c),k(c,z)\}=k(x,c)$,
 since $k(c,z)=k_{\text{max}}$. Thus $k(z,x)=k(c,x)$ must hold and consequently $k'(z,x) = k(z,x)$.
\end{proof}

Note that a hierarchy inducing a specific strong kernel is not unique: 
Adjacent inner vertices with the same weight can be merged, and vertices with 
just one child can be removed without changing the induced kernel. 
Combining Lemmas~\ref{prop:strong_hierarchy:left} and~\ref{prop:strong_hierarchy:right}
we obtain the following result.
\begin{theorem}\label{thm:strong_hierarchy}
 A kernel $k$ on \X is strong if and only if it is induced by a hierarchy on \X.
\end{theorem}

As a consequence of the above theorem the number of values a strong kernel 
on $n$ objects may take is bounded by the number of vertices in a binary tree 
with $n$ leaves, i.e., for every strong kernel $k$ on \X we have $|\img(k)| \leq 2|\X|-1$.
The Dirac kernel is a common example of a strong kernel, in fact, every kernel 
$k : \X \times \X \to \bbRnn$ with $|\img(k)|=2$ is strong.

The definition of a strong kernel and its relation to hierarchies is reminiscent 
of related concepts for distances: A metric $d$ on \X is an \emph{ultrametric}
if $d(x,y)\leq \max\{d(x,z),d(z,y)\}$ for all $x,y,z\in \X$. 
For every ultrametric $d$ on \X there is a rooted tree $T$ with leaves \X and 
edge weights, such that 
\begin{inparaenum}[(i)]
 \item $d$ is the path length between leaves in $T$,
 \item the path lengths from a leaf to the root are all equal.
\end{inparaenum}
Indeed, every ultrametric can be embedded into a Hilbert space~\cite{Ismagilov1997}
and thus the associated inner product is a valid kernel. 
Moreover, it can be shown that this inner product always is a strong kernel.
However, the concept of strong kernels is more general: there are strong kernels 
$k$ such that the associated kernel metric $d_k(x,y) = \norm{\phi(x)-\phi(y)}$
is not an ultrametric.
The distinction originates from the self-similarities, which in strong kernels, 
can be arbitrary provided that they fulfil $k(x,x)\geq k(x,y)$
for all $x,y$ in \X. 
This degree of freedom is lost when considering distances. 
If we require all self-similarities of a strong kernel to be equal, then the 
associated kernel metric always is an ultrametric.
Consequently, strong kernels correspond to a superset of ultrametrics.
We explicitly define a feature space for general strong kernels in the following.

\subsection{Feature maps of strong kernels}\label{sec:strong:feature_maps}

We use the property that every strong kernel is induced by a hierarchy to derive
feature vectors for strong kernels.
Let $(T,w)$ be a hierarchy on \X that induces the strong kernel $k$. We define 
the additive weight function $\omega : V(T) \to \bbRnn$ as
$\omega(v) = w(v)-w(p(v))$ and $\omega(r) = w(r)$ for the root $r$. 
Note that the property of a hierarchy assures that the difference is non-negative.
For $v \in V(T)$ let $P(v) \subseteq V(T)$ denote the vertices in $T$ on
the path from $v$ to the root $r$.

We consider the mapping $\phi : \X \to \bbR^t$, where $t = |V(T)|$ and
the components indexed by $v \in V(T)$ are
\begin{equation*}
 [\phi(x)]_v = 
 \begin{cases}
   \sqrt{\omega(v)},  & \text{if } v \in P(x) \\
   0,                 & \text{otherwise.}
 \end{cases}
\end{equation*}
\begin{proposition}\label{prop:strong:feature_map}
 Let $k$ be a strong kernel on \X. The function $\phi$ defined as 
 above is a feature map of $k$, i.e., $k(x,y) = \phi(x)^\tp \phi(y)$ for all 
 $x,y \in \X$.
\end{proposition}
\begin{proof}
 Given arbitrary $x,y \in \X$ and let $c={\rm LCA}(x,y)$.
 The dot product yields 
 \begin{equation*}
  \phi(x)^\tp \phi(y) = 
  \sum_{v \in V(T)} [\phi(x)]_v [\phi(y)]_v =
  \sum_{v \in P(c)} \sqrt{\omega(v)}^2 = 
  w(c) = k(x,y),
 \end{equation*}
 since according to the definition the only non-zero products contributing to the 
 sum over $v \in V(T)$ are those in $P(x) \cap P(y) = P(c)$. 
\end{proof}

Figure~\ref{fig:example} shows an example of a strong kernel, an associated 
hierarchy and the derived feature vectors.
As a consequence of Theorem~\ref{thm:strong_hierarchy} and 
Proposition~\ref{prop:strong:feature_map}, strong kernels according to Definition~\ref{def:strong} are indeed valid kernels.

\begin{figure}
  \centering
  \null\hfill
  \subfigure[Kernel matrix]{\label{fig:example:kernel_matrix}
    \includegraphics{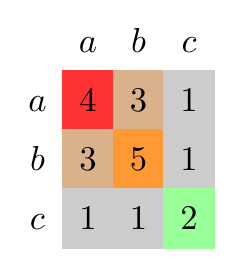}
  }\hfill
  \subfigure[Hierarchy]{\label{fig:example:hierarchy}
    \includegraphics[scale=.9]{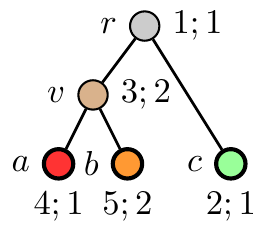}
  }\hfill
  \subfigure[Feature vectors]{\label{fig:example:vectors}
    \includegraphics[scale=.9]{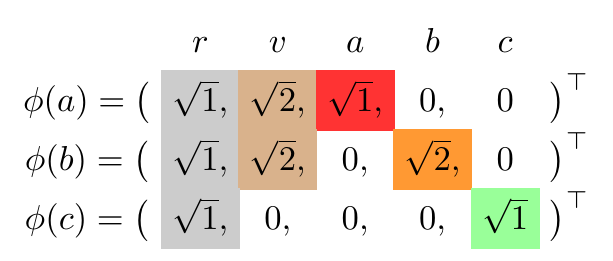}
  }
  \hfill\null
  \caption{The matrix of a strong kernel on three 
    objects~\subref{fig:example:kernel_matrix} induced by the
    hierarchy~\subref{fig:example:hierarchy} and the derived feature 
    vectors~\subref{fig:example:vectors}. A vertex $u$ 
    in~\subref{fig:example:hierarchy} is annotated by its weights 
    $w(u);\omega(u)$.
  }
  \label{fig:example}
\end{figure}

\section{Valid kernels from optimal assignments}
We consider the function $K_\Assign^k$ on $[\X]^n$ according to Eq.~\eqref{eq:assignment_kernel} 
under the assumption that the base kernel $k$ is strong.
Let $(T,w)$ be a hierarchy on \X which induces $k$.
For a vertex $v \in V(T)$ and a set $X \subseteq \X$, we denote by $X_v$ the 
subset of $X$ that is contained in the subtree rooted at $v$.
We define the histogram $H^k$ of a set $X \in [\X]^n$ w.r.t.\ the strong base 
kernel $k$ as $H^k(X) = \sum_{x \in X} \phi(x) \circ \phi(x)$, where $\phi$ is 
the feature map of the strong base kernel according to 
Section~\ref{sec:strong:feature_maps} and $\circ$ denotes the element-wise 
product. Equivalently, $[H^k(X)]_v = \omega(v) \cdot |X_v|$ for $v \in V(T)$.
The \emph{histogram intersection kernel}~\cite{Swain1991} is defined as
$K_\sqcap(\vec{g},\vec{h}) = \sum_{i=1}^t \min\{[\vec{g}]_i,[\vec{h}]_i\}$,
$t \in \bbN$, and known to be a valid kernel on $\bbR^t$~\cite{Barla2003,Boughorbel2005}.

\begin{theorem}\label{thm:strong_hist}
 Let $k$ be a strong kernel on \X and the histograms $H^k$ defined as 
 above, then $K_\Assign^k(X,Y) = K_\sqcap\left(H^k(X),H^k(Y)\right)$ for all $X,Y \in [\X]^n$.
\end{theorem}
\begin{proof}
 Let $(T,w)$ be a hierarchy inducing the strong base kernel $k$.
 We rewrite the weight of an assignment $B$ as sum of weights of vertices in $T$.
 Since
 \begin{align*}
  k(x,y) = 
  w({\rm LCA}(x,y)) = 
  \sum_{\mathclap{v \in P(x) \cap P(y)}} \omega(v),
  \text{ we have }\enskip 
  W(B) &= \sum_{\mathclap{(x,y) \in B}} k(x,y) = 
 \sum_{\mathclap{v \in V(T)}} c_v \cdot \omega(v),
 \end{align*}
 where $c_v$ counts how often $v$ appears simultaneously in $P(x)$ and $P(y)$
 in total for all $(x,y) \in B$.
 For the histogram intersection kernel we obtain
 \begin{align*}
  K_\sqcap(H^k(X),H^k(Y)) 
  = \sum_{\mathclap{v \in V(T)}} \min\{\omega(v) \cdot |X_v|, \omega(v) \cdot |Y_v|\}
  = \sum_{\mathclap{v \in V(T)}} \min\{|X_v|, |Y_v|\} \cdot \omega(v).
 \end{align*}
 Since every assignment $B \in \Assign(X,Y)$ is a bijection, each $x \in X$ and 
 $y \in Y$ appears only once in $B$ and $c_v \leq \min \{ |X_v|, |Y_v| \}$ follows. 

 It remains to show that the above inequality is tight for an optimal assignment.
 We construct such an assignment by the following greedy approach: We perform a bottom-up
 traversal on the hierarchy starting with the leaves. 
 For every vertex $v$ in the hierarchy we arbitrarily pair the objects in $X_v$ 
 and $Y_v$ that are not yet contained in the assignment. Note that no element
 in $X_v$ has been assigned to an element in $Y \setminus Y_v$, and no element 
 in $Y_v$ to an element from $X \setminus X_v$. Hence, at every vertex $v$ we 
 have $c_v = \min\{|X_v|, |Y_v|\}$ vertices from $X_v$ assigned to vertices in 
 $Y_v$. 
\end{proof}

Figure~\ref{fig:example:hist} illustrates the relation between the optimal 
assignment kernel employing a strong base kernel and the histogram intersection
kernel.
Note that a vertex $v \in V(T)$ with $\omega(v)=0$ does not contribute to the
histogram intersection kernel and can be omitted. 
In particular, for any two objects $x_1,x_2 \in \X$ with $k(x_1,y)=k(x_2,y)$ 
for all $y \in \X$ we have $\omega(x_1) = \omega(x_2) = 0$. There is no need
to explicitly represent such leaves in the hierarchy, yet their multiplicity must
be considered to determine the number of leaves in the subtree rooted at an
inner vertex, cf. Fig.~\ref{fig:example}, \ref{fig:example:hist}.

\begin{corollary}\label{cor:strong_valid}
 If the base kernel $k$ is strong, then the function $K_\Assign^k$ is a valid kernel.
\end{corollary}

Theorem~\ref{thm:strong_hist} implies not only that optimal assignments give
rise to valid kernels for strong base kernels, but also allows to compute 
them by histogram intersection. Provided that the hierarchy is known, bottom-up 
computation of histograms and their intersection can both be performed in linear 
time, while the general Hungarian method would require cubic time to solve the 
assignment problem~\cite{Burkard2012}.

\begin{corollary}
 Given a hierarchy inducing $k$, $K_\Assign^k(X,Y)$ can be computed in time 
 $\mathcal{O}(|X|+|Y|)$.
\end{corollary}

\begin{figure}
  \centering
  \null\hfill
  \subfigure[Assignment problem]{\label{fig:example:assignment_problem}
    \includegraphics{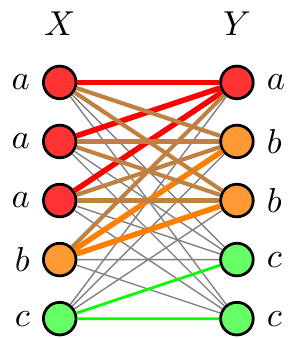}
  }\hfill
  \subfigure[Histograms]{\label{fig:example:histograms}%
    \includegraphics{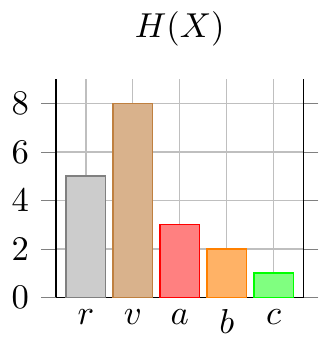}%
    \includegraphics{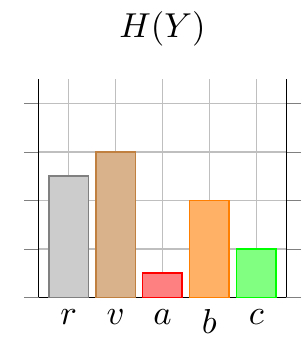}%
  }
  \hfill\null
  \caption{
    An assignment instance~\subref{fig:example:assignment_problem} for 
    $X,Y \in [\X]^5$ and the derived histograms~\subref{fig:example:histograms}.
    The set $X$ contains three distinct vertices labelled $a$ and the set $Y$ two
    distinct vertices labelled $b$ and $c$. Taking the multiplicities into account 
    the histograms are obtained from the hierarchy of the base kernel $k$ depicted in Fig.~\ref{fig:example}.
    The optimal assignment yields a value of $K_\Assign^k(X,Y)=15$, where grey,
    green, brown, red and orange edges have weight $1$, $2$, $3$, $4$ and $5$, 
    respectively.
    The histogram intersection kernel gives
    $K_\sqcap(H^k(X),H^k(Y))= \min\{5,5\}+\min\{8,6\}+\min\{3,1\}+\min\{2,4\}+\min\{1,2\}=15$.       
  }
  \label{fig:example:hist}
\end{figure}

\section{Graph kernels from optimal assignments}
The concept of optimal assignment kernels is rather general and can be applied
to derive kernels on various structures.
In this section we apply our results to obtain novel graph kernels, i.e., kernels
of the form $K : \G \times \G \to \bbR$, where \G denotes the set of graphs.
We assume that every vertex $v$ is equipped with a categorical label given by $\lab(v)$. 
Labels typically arise from applications, e.g., in a graph representing a chemical 
compound the labels may indicate atom types.

\subsection{Optimal assignment kernels on vertices and edges} \label{sec:baseline}

As a baseline we propose graph kernels on vertices and edges.
The \emph{vertex optimal assignment kernel} (V-OA) is defined as 
$K(G,H) = K_\Assign^k(V(G), V(H))$, where $k$ is the Dirac kernel on vertex 
labels.
Analogously, the \emph{edge optimal assignment kernel} (E-OA) is given by
$K(G,H) = K_\Assign^k(E(G), E(H))$, where we define $k(uv, st) = 1$ if at least
one of the mappings $(u \mapsto s, v \mapsto t)$ and $(u \mapsto t, v \mapsto s)$
maps vertices with the same label only; and $0$ otherwise.
Since these base kernels are Dirac kernels, they are strong and, consequently, 
V-OA and E-OA are valid kernels.

\subsection{Weisfeiler-Lehman optimal assignment kernels} \label{sec:wl}

\emph{Weisfeiler-Lehman kernels} are based on iterative vertex colour refinement 
and have been shown to provide state-of-the-art prediction performance in 
experimental evaluations~\cite{Shervashidze2011}.
These kernels employ the classical $1$-dimensional Weisfeiler-Lehman heuristic 
for graph isomorphism testing and consider subtree patterns encoding the
neighbourhood of each vertex up to a given distance. 
For a parameter $h$ and a graph $G$ with initial labels $\lab$, a sequence 
$(\lab_0,\dots,\lab_h)$ of refined labels referred to as \emph{colours} is computed,
where $\lab_0 = \lab$ and $\lab_i$ is obtained from $\lab_{i-1}$ by the following 
procedure:
Sort the multiset of colours $\{\lab_{i-1}(u) :\, vu \in E(G)\}$ for every 
vertex $v$ lexicographically to obtain a unique sequence of colours and add 
$\lab_{i-1}(v)$ as first element. Assign a new colour $\lab_i(v)$ to every vertex 
$v$ by employing a one-to-one mapping from sequences to new colours. 
Figure~\ref{fig:wl:graph} illustrates the refinement process.
\begin{figure}
  \centering
  \null\hfill
  \subfigure[Graph $G$ with refined colours]{\label{fig:wl:graph}
    \includegraphics{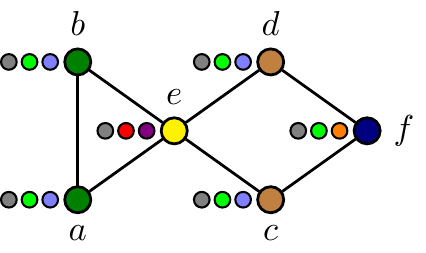}
  }\hfill
  \subfigure[Feature vector]{\label{fig:wl:colo_count}%
    \includegraphics{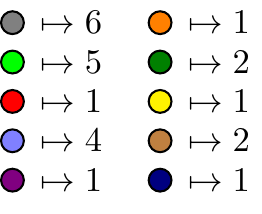}
  }\hfill
  \subfigure[Associated hierarchy]{\label{fig:wl:hierarchy}%
    \includegraphics{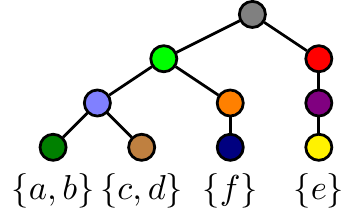}
  }
  \hfill\null
  \caption{
    A graph $G$ with uniform initial colours $\lab_0$ and refined colours 
    $\lab_i$ for $i \in \{1,\dots,3\}$~\subref{fig:wl:graph}, 
    the feature vector of $G$ for the Weisfeiler-Lehman subtree kernel~\subref{fig:wl:colo_count}
    and the associated hierarchy~\subref{fig:wl:hierarchy}.
    Note that the vertices of $G$ are the leaves of the hierarchy, although
    not shown explicitly in Fig.~\ref{fig:wl:hierarchy}.
  }
  \label{fig:wl}
\end{figure}
The \emph{Weisfeiler-Lehman subtree kernel} (WL) 
counts the vertex colours two graphs have in common in the first $h$ refinement 
steps and can be computed by taking the dot product of feature vectors, where 
each component counts the occurrences of a colour, see Fig.~\ref{fig:wl:colo_count}.

We propose the \emph{Weisfeiler-Lehman optimal assignment kernel} (WL-OA),
which is defined on the vertices like OA-V, but employs the non-trivial base kernel 
\begin{equation}\label{eq:strong:wl}
  k(u,v) = \sum_{i=0}^h k_\delta(\lab_i(u), \lab_i(v)).
\end{equation}
This base kernel corresponds to the number of matching colours in the refinement 
sequence.
More intuitively, the base kernel value reflects to what extent the two vertices 
have a similar neighbourhood.

Let \V be the set of all vertices of graphs in \G, we show that the refinement 
process defines a hierarchy on \V, which induces the base kernel of 
Eq.~\eqref{eq:strong:wl}.
Each vertex colouring $\lab_i$ naturally partitions \V into colour classes, 
i.e., sets of vertices with the same colour.
Since the refinement takes the colour $\lab_i(v)$ of a vertex $v$ into account 
when computing $\lab_{i+1}(v)$, the implication 
$\lab_i(u) \neq \lab_i(v) \Rightarrow \lab_{i+1}(u) \neq \lab_{i+1}(v)$ holds
for all $u,v \in \V$. 
Hence, the colour classes induced by $\lab_{i+1}$ are at least as fine as those 
induced by $\lab_{i}$. Moreover, the sequence $(\lab_i)_{0 \leq i \leq h}$ gives 
rise to a family of nested subsets, which can naturally be represented by a
hierarchy $(T,w)$, see Fig.~\ref{fig:wl:hierarchy} for an illustration.
When assuming $\omega(v)=1$ for all vertices $v \in V(T)$, the hierarchy induces
the kernel of Eq.~\eqref{eq:strong:wl}. 
We have shown that the base kernel is strong and it follows from 
Corollary~\ref{cor:strong_valid} that WL-OA is a valid kernel. Moreover,
it can be computed from the feature vectors of the Weisfeiler-Lehman subtree 
kernel in linear time by histogram intersection, cf. Theorem~\ref{thm:strong_hist}.

\section{Experimental evaluation}
We report on the experimental evaluation of the proposed graph kernels derived 
from optimal assignments and compare with state-of-the-art convolution kernels.

\subsection{Method and Experimental Setup}
We performed classification experiments using the $C$-SVM implementation 
LIBSVM~\cite{Chang2011}.
We report mean prediction accuracies and standard deviations obtained by 
$10$-fold cross-validation repeated $10$ times with random fold assignment.
Within each fold all necessary parameters were selected by cross-validation 
based on the training set. This includes the regularization parameter $C$, kernel
parameters where applicable and whether to normalize the kernel matrix.
All kernels were implemented in Java and experiments were conducted using 
Oracle Java v1.8.0 on an Intel Core i7-3770 CPU at 3.4GHz (Turbo Boost disabled)
with 16GB of RAM using a single processor only.

\paragraph{Kernels.}
As a baseline we implemented the \emph{vertex kernel} (V) and \emph{edge 
kernel} (E), which are the dot products on vertex and edge label 
histograms, respectively, where an edge label consist of the labels of its
endpoints. V-OA and E-OA are the related optimal assignment kernels as 
described in Sec.~\ref{sec:baseline}. For the Weisfeiler-Lehman kernels WL and 
WL-OA, see Section~\ref{sec:wl}, the parameter $h$ was chosen from $\{0,...,7\}$.
In addition we implemented a \emph{graphlet kernel} (GL) and the shortest-path 
kernel (SP)~\cite{Borgwardt2005}.
GL is based on connected subgraphs with three vertices taking labels into account 
similar to the approach used in~\cite{Shervashidze2011}. For SP we used the Dirac 
kernel to compare path lengths and computed the kernel by explicit feature maps, cf.~\cite{Shervashidze2011}. 
Note that all kernels not identified as optimal assignment kernels by the suffix 
OA are convolution kernels.

\paragraph{Data sets.}
We tested on widely-used graph classification benchmarks from different domains, 
cf. \cite{Borgwardt2005a,Vishwanathan2010,Shervashidze2011,Yanardag2015a}: 
\textsc{Mutag}, \textsc{PTC-MR}, \textsc{NCI1} and \textsc{NCI109} are graphs 
derived from small molecules, \textsc{Proteins}, \textsc{D\&D} and \textsc{Enzymes} 
represent macromolecules, and \textsc{Collab} and \textsc{Reddit} are derived 
from social networks.\footnote{The data sets, further references and statistics 
are available from \url{http://graphkernels.cs.tu-dortmund.de}.}
All data sets have two class labels except \textsc{Enzymes} and \textsc{Collab}, 
which are divided into six and three classes, respectively.
The social network graphs are unlabelled and we considered all vertices uniformly 
labelled. All other graph data sets come with vertex labels.
Edge labels, if present, were ignored since they are not supported by all graph 
kernels under comparison.

\subsection{Results and discussion}
Table~\ref{tab:accuracies} summarizes the classification accuracies. We observe
that optimal assignment kernels on most data sets improve over the prediction
accuracy obtained by their convolution-based counterpart. The only distinct 
exception is \textsc{Mutag}.
The extent of improvement on the other data sets varies, but is in particular
remarkable for \textsc{Enzymes} and \textsc{Reddit}.
This indicates that optimal assignment kernels provide a more valid notion
of similarity than convolution kernels for these classification tasks.
The most successful kernel is WL-OA, which almost consistently 
improves over WL and performs best on seven of the nine data sets. 
WL-OA provides the second best accuracy on \textsc{D\&D} and ranks in the middle 
of the field for \textsc{Mutag}. For these two data set the difference in 
accuracy between the kernels is small and even the baseline kernels perform 
notably well.

The time to compute the quadratic kernel matrix was less that one minute for all 
kernels and data sets with exception of SP on \textsc{D\&D} (29\,min) and 
\textsc{Reddit} (2\,h) as well as GL on \textsc{Collab} (28\,min).
The running time to compute the optimal assignment kernels by histogram 
intersection was consistently on par with the running time required for the 
related convolution kernels and orders of magnitude faster than their 
computation by the Hungarian method. 

\setlength\tabcolsep{3.7pt}
\newcommand{\win}[1]{$\hspace{-0.3mm}$\textbf{#1}}
\newcommand{\sd}[1]{\scriptsize{$\pm$#1}}
\begin{table*}\small
\begin{center}
  \caption{Classification accuracies and standard deviations on graph data sets 
  representing small molecules, macromolecules and social networks.}
  \label{tab:accuracies}
  \begin{tabular}{@{}lccccccccc@{}}\toprule
    \multirow{3}{*}{\textbf{Kernel}}    &\multicolumn{9}{c}{\textbf{Data Set}}\\\cmidrule{2-10}
                     & \textsc{Mutag}   & \textsc{PTC-MR}  & \textsc{NCI1}    & \textsc{NCI109}  & \textsc{Proteins} & \textsc{D\&D}    & \textsc{Enzymes}  & \textsc{Collab}  & \textsc{Reddit}  \\\midrule
    \textsc{V}       &  85.4\sd{0.7}    &  57.8\sd{0.9}    &  64.6\sd{0.1}    &  63.6\sd{0.2}    &  71.9\sd{0.4}     &  78.2\sd{0.4}    &  23.4\sd{1.1}     & 56.2\sd{0.0}     & 75.3\sd{0.1}     \\
    \textsc{V-OA}    &  82.5\sd{1.1}    &  56.4\sd{1.8}    &  65.6\sd{0.3}    &  65.1\sd{0.4}    &  73.8\sd{0.5}     &  78.8\sd{0.3}    &  35.1\sd{1.1}     & 59.3\sd{0.1}     & 77.8\sd{0.1}     \\\midrule
    \textsc{E}       &  85.2\sd{0.6}    &  57.3\sd{0.7}    &  66.2\sd{0.1}    &  64.9\sd{0.1}    &  73.5\sd{0.2}     &  78.3\sd{0.5}    &  27.4\sd{0.8}     & 52.0\sd{0.0}     & 75.1\sd{0.1}     \\
    \textsc{E-OA}    &  81.0\sd{1.1}    &  56.3\sd{1.7}    &  68.9\sd{0.3}    &  68.7\sd{0.2}    &  74.5\sd{0.6}     &  79.0\sd{0.4}    &  37.4\sd{1.8}     & 68.2\sd{0.3}     & 79.8\sd{0.2}     \\\midrule 
    \textsc{WL}      &\win{86.0}\sd{1.7}&  61.3\sd{1.4}    &  85.8\sd{0.2}    &  85.9\sd{0.3}    &  75.6\sd{0.4}     &  79.0\sd{0.4}    &  53.7\sd{1.4}     & 79.1\sd{0.1}     & 80.8\sd{0.4}     \\
    \textsc{WL-OA}   &  84.5\sd{1.7}    &\win{63.6}\sd{1.5}&\win{86.1}\sd{0.2}&\win{86.3}\sd{0.2}&\win{76.4}\sd{0.4} &  79.2\sd{0.4}    &\win{59.9}\sd{1.1} &\win{80.7}\sd{0.1}&\win{89.3}\sd{0.3}\\\midrule
    \textsc{GL}      &  85.2\sd{0.9}    &  54.7\sd{2.0}    &  70.5\sd{0.2}    &  69.3\sd{0.2}    &  72.7\sd{0.6}     &\win{79.7}\sd{0.7}&  30.6\sd{1.2}     & 64.7\sd{0.1}     & 60.1\sd{0.2}     \\
    \textsc{SP}      &  83.0\sd{1.4}    &  58.9\sd{2.2}    &  74.5\sd{0.3}    &  73.0\sd{0.3}    &  75.8\sd{0.5}     &  79.0\sd{0.6}    &  42.6\sd{1.6}     & 58.8\sd{0.2}     & 84.6\sd{0.2}     \\
    \bottomrule
  \end{tabular}
\end{center}
\end{table*}

\section{Conclusions and future work}
We have characterized the class of strong kernels leading to valid optimal 
assignment kernels and derived novel effective kernels for graphs. 
The reduction to histogram intersection makes efficient computation possible and
known speed-up techniques for intersection kernels can directly be 
applied (see, e.g., \cite{Vedaldi2012} and references therein).
We believe that our results may form the basis for the design of new kernels, 
which can be computed efficiently and adequately measure similarity.

\begin{small}
\subsubsection*{Acknowledgments}

N.~M.~Kriege is supported by the German Science Foundation (DFG) within the Collaborative Research Center SFB 876 ``Providing Information by 
Resource-Constrained Data Analysis'', project A6 ``Resource-efficient Graph Mining''.
P.-L.~Giscard is grateful for the financial support provided by the Royal Commission for the Exhibition of 1851.

\bibliographystyle{abbrv}
\bibliography{lit_short}
\end{small}

\end{document}